%% file: ijcai26_final.tex
\documentclass{article}
\usepackage{ijcai26}

\usepackage{times}
\usepackage{soul}
\usepackage{url}
\usepackage[hidelinks]{hyperref}
\usepackage[utf8]{inputenc}
\usepackage[small]{caption}
\usepackage{graphicx}
\usepackage{amsmath}
\usepackage{amsthm}
\usepackage{booktabs}
\usepackage{arydshln}
\usepackage{multirow}
\usepackage{algorithm}
\usepackage{algorithmic}
\usepackage[switch]{lineno}
\usepackage{comment}
\usepackage{inconsolata}
\usepackage{listings}
\usepackage{amssymb}
\usepackage{xspace}
\usepackage{enumitem}
\usepackage{subcaption}
\usepackage[dvipsnames]{xcolor}

\newcommand{\modelname}{DProp\xspace}
\newcommand{\neuralmodelname}{NDProp\xspace}

\newcommand{\edit}{}

\newcommand{\Tr}{T}
\newcommand{\Fa}{F}
\newcommand{\Un}{U}

\newcommand{\Cp}[4]{C_{{#1},{\mathcal{A}\setminus#2}}^{#3}(#4)}
\newcommand{\Cpp}[4]{C_{{#1},{\mathcal{A}_P\setminus#2}}^{#3}(#4)}
\newcommand{\Cpnorm}[4]{C_{{#1},{#2}}^{#3}(#4)}
\newcommand{\FCp}[3]{\mathcal{C}_{{#1},{#2}}(#3)}

\theoremstyle{plain}
\newtheorem{theorem}{Theorem}
\newtheorem{lemma}{Lemma}
\newtheorem{proposition}{Proposition}

\newtheorem{definition}{Definition}

\newtheorem{remark}{Remark}

\newcommand{\citeNBYB}[1]{\citeauthor{#1}~\shortcite{#1}}

\newcounter{myenumctr}

\title{Neural Decision--Propagation for Answer Set Programming}

\author{
Thomas Eiter$^1$
\and
Katsumi Inoue$^{2}$\and
Sota Moriyama$^{3,2}$
\affiliations
$^1$Vienna University of Technology (TU Wien), Austria\\
$^2$National Institute of Informatics, Japan\\
$^3$The Graduate University for Advanced Studies, SOKENDAI, Japan\\
\emails
thomas.eiter@tuwien.ac.at,
\{inoue, sotam\}@nii.ac.jp
}

\begin{document}

\maketitle

\begin{abstract}
Integration of Answer Set Programming (ASP) with neural networks has emerged as a promising tool in Neuro-symbolic AI. While existing approaches extend the capabilities of ASP to real world domains, their reasoning pipelines depend on classical solvers, which is a bottleneck for scalability. To tackle this problem, we propose a new method to compute stable models, called decision--propagation (DProp), which alternates falsity decisions and truth propagations. Successful DProp computations are shown to capture the stable model semantics. We then develop Neural DProp (NDProp), a differentiable extension of DProp with neural computation for decisions and fuzzy evaluation for propagations. We evaluate the capabilities of NDProp for learning decision heuristics as well as neuro-symbolic integration, and compare it with existing neuro-symbolic approaches. The results show that NDProp can learn to efficiently compute stable models, and it improves accuracy and scalability on neuro-symbolic benchmarks.
\footnote{This is the full version (with appendix) of a paper appearing at IJCAI-ECAI 2026. The code is available at \url{https://github.com/sotam2369/NDProp}.}
\end{abstract}

\section{Introduction}
Neuro-symbolic AI aims to develop systems that bring together robust learning in neural networks with reasoning via symbolic representations~\cite{DBLP:journals/air/GarcezL23,DBLP:journals/pami/WangYW25}.
There has been a growing interest in using logic programming 
as a general symbolic reasoning framework for such integrated systems~\cite{DBLP:journals/jair/EvansG18,DBLP:conf/ijcai/YangIL20,DBLP:journals/ai/ManhaeveDKDR21}.
Answer Set Programming (ASP), a declarative programming paradigm widely used for knowledge representation and reasoning tasks~\cite{DBLP:conf/aaai/Lifschitz08,DBLP:journals/cacm/BrewkaET11}, is used in such systems.
Neuro-symbolic integration in ASP has been attempted in various domains such as Visual Question Answering~\cite{DBLP:journals/tplp/EiterHOP22}, compliance checking~\cite{DBLP:journals/tplp/BarbaraGLMQRR23} and many more~\cite{nesy_asp_survey}.

One prominent approach to accomplishing neuro-symbolic integration with ASP is to train neural networks with symbolic knowledge expressed with ASP.  
NeurASP~\cite{DBLP:conf/ijcai/YangIL20} integrates neural atoms into ASP programs, allowing neural network outputs to be directly incorporated into symbolic reasoning and learning.
SLASH~\cite{DBLP:journals/jair/SkryaginODK23} extends this line of work with neural-probabilistic predicates and removes low-probability predicate outcomes before ASP solving, thereby reducing the number of stable models that must be considered.
However, these approaches largely follow a hybrid design in which neural components and symbolic ASP solvers remain separated by an interface: continuous neural predictions must be translated into symbolic atoms/literals, constraints, or probabilistic facts before reasoning can be performed. 
This introduces a combinatorial grounding problem, as the system must reason over multiple candidate associations between perceptual entities and symbolic objects before being processed by an ASP solver.

To connect neural outputs and symbolic representations smoothly,  it would be desirable to compute answer sets in continuous domains. Then, to make this connection easily, we first propose a novel method for computing stable models called \emph{decision--propagation} (\modelname).
A \modelname run iteratively alternates between the following two phases: (i) decision, which selects undecided atoms to commit to false non-deterministically, and (ii) propagation, which updates the current partial assignment by deriving consequences of the program under the current assumptions.
We then prove that successful \modelname computations (runs that end without contradictions) capture the stable model semantics.

Subsequently, we propose \emph{Neural \modelname} (\neuralmodelname), a differentiable extension of \modelname.
\neuralmodelname uses neural networks for falsity decisions and fuzzy logic operators for propagations.
As \neuralmodelname is identical to \modelname in the crisp case, it is able to prove stability when the outputs are binary under a certain condition, making it easy to train in an unsupervised manner.
These make the integration of \neuralmodelname into neuro-symbolic systems seamless, while still preserving a clear connection to classical stable model semantics.
Furthermore, the learned decision steps and parallel capabilities via GPUs make this approach significantly more scalable.

To demonstrate the effectiveness of \neuralmodelname, we conduct three 
kinds of experiments: (1) random ASP solving to determine whether \neuralmodelname can learn the solving dynamics 
for given ASP problems,
(2) neuro-symbolic integration to determine whether using \neuralmodelname has merits over existing approaches that use classical solvers, and 
(3) fuzzy reasoning to determine whether \neuralmodelname can reason and find stable models given fuzzy values as input.
The results show that \neuralmodelname performs better than the respective baselines; 
it is superior with respect to accuracy, and has much higher 
inference speed 
than  existing ASP-based neuro-symbolic approaches.

In summary, our contributions are:
\begin{itemize}
  \item We introduce a novel method for computing stable models called \modelname, and prove that a successful \modelname computation captures stable model semantics.
  \item We develop \neuralmodelname, a differentiable extension of \modelname that uses neural decisions and fuzzy propagations, and show that it coincides with \modelname at the crisp case.
  \item We analyze the capabilities of \neuralmodelname in learning the solving dynamics of random ASP problems, and attempt to integrate \neuralmodelname into neuro-symbolic frameworks.
  \item We demonstrate that \neuralmodelname performs better than respective baselines, specifically with accuracy and inference speed.
\end{itemize}
Consequently, we propose a fully-differentiable answer set solver based on novel fixpoint semantics, yielding improved scalability compared to previous neural approaches to ASP.

\section{Preliminaries}
\subsection{Answer Set Programming}
Answer Set Programming (ASP) is a declarative paradigm based on non-monotonic logic and is widely used for knowledge representation and reasoning~\cite{DBLP:conf/aaai/Lifschitz08,DBLP:journals/cacm/BrewkaET11}.
A \emph{rule} $r$ in a \emph{(normal logic) program} $P$ has the form
\begin{equation}
\label{eq:rule}
      a \leftarrow b_1, \ldots, b_m, not\, c_1, \ldots, not\, c_l,
\end{equation}
where $a$, $b_i$, and $c_j$ are atoms and $not$ denotes default negation.
For a rule $r$ of the form~(\ref{eq:rule}), the atom $a$ is the \emph{head} of the rule denoted as ${H}(r)$, and $b_1, \ldots, b_m, not\, c_1, \ldots, not\, c_l$ is the \emph{body} of the rule denoted as ${B}(r)$.
The sets ${B}^+(r)=\{b_1,\ldots,b_m\}$ and ${B}^-(r)=\{c_1,\ldots,c_l\}$ represent the positive and default-negated body atoms of $r$, respectively.
An \emph{interpretation} $I$, i.e., a subset of the set $\mathcal{A}$ of all atoms, is a
\emph{model}\/ of $P$ if it satisfies each rule $r$ of $P$, i.e., $B^+(r)\subseteq I$ and $B^-(r)\cap I=\emptyset$ imply $H(r)\in I$.
Furthermore, $I$ is a \emph{stable model} (\emph{answer set}) of $P$ if it is the least model of the \emph{Gelfond--Lifschitz reduct} $P^I$, which is obtained from $P$ by deleting every rule whose negative body intersects $I$ and removing all negative literals from the remaining rules.

Given two interpretations $I$, $J$ and a normal logic program $P$, an immediate consequence operator with respect to $J$ \cite{DBLP:journals/jcss/Gelder93} is defined as
\[
C_{P,J}(I)
=
\left\{
H(r) \in \mathcal{A} \mid
B^+(r) \subseteq I,B^-(r) \cap J = \emptyset
\right\}.
\]
We define $C_{P,J}^0(I)=I$ and $C_{P,J}^{k+1}(I)=C_{P,J}(C_{P,J}^{k}(I))$ for integer $k\,(\ge 0)$. 
When $P^+$ is a \emph{definite program}, i.e., a program whose rule contains no default negation $(l=0)$,  
for fixed $J$ there is a least fixpoint of $C_{P^+,J}$ starting from $\emptyset$, which is the least model of $P^+$, denoted as $C^\infty_{P^+,J}(\emptyset)$~\cite{DBLP:journals/jacm/EmdenK76}.  
Then, the next lemma holds.
\begin{lemma}
\label{lem:cp_sm}
    $S$ is a stable model of $P$ iff $S=\Cpnorm{P}{S}{\infty}{\emptyset}$.
\end{lemma}

\begin{proof}
\edit{Let $\mathcal{T}_P$ be the immediate consequence operator of a definite program $P$, and write
$\mathcal{T}_P^\infty(\emptyset)$ for the result of iterating $\mathcal{T}_P$ from $\emptyset$ to its least fixpoint. By definition, $S$ is a stable model of $P$
iff $S=\mathcal{T}_{P^{S}}^\infty(\emptyset)$, where $P^{S}$ is the Gelfond--Lifschitz reduct.
For normal programs, iterating $C_{P,S}$ from $\emptyset$ coincides with iterating $\mathcal{T}_{P^{S}}$,
so $C_{P,S}^{\infty}(\emptyset)=\mathcal{T}_{P^{S}}^\infty(\emptyset)$. Hence $S$ is stable iff $S=\Cpnorm{P}{S}{\infty}{\emptyset}$.}
\end{proof}

\subsection{Fuzzy Interpretations}
In many real-world settings it is necessary to reason with vague or incomplete information.
{Fuzzy Answer Set Programming} (FASP) extends ASP by interpreting atoms with a truth degree in the interval $[0,1]$ and evaluating rule bodies, while adopting stable-model-based semantics for fuzzy interpretations~\cite{DBLP:conf/jelia/NieuwenborghCV06,DBLP:journals/tplp/AlvianoP13}.

Let $\mathcal{A}$ be a set of propositional atoms.
A \emph{fuzzy interpretation} is a mapping $I:\mathcal{A}\to[0,1]$ assigning each atom a degree of truth.
To evaluate rule bodies, FASP fixes a \emph{t-norm} $\otimes:[0,1]^2\to[0,1]$ (fuzzy conjunction), i.e., a commutative, associative and monotone operator with unit~$1$.
Typical t-norms include G\"odel, Product, and \L{}ukasiewicz.
Similarly, \emph{t-conorm} $\oplus:[0,1]^2\to[0,1]$ represents fuzzy disjunction.

For a rule $r$ of the form~(\ref{eq:rule}), the degree to which the body is {satisfied} under $I$ is
\[
  I({B}(r))
  \;:=\;
  \Bigl(\bigotimes_{i=1}^m I(b_i)\Bigr)
  \;\otimes\;
  \Bigl(\bigotimes_{j=1}^l \bigl(1-I(c_j)\bigr)\Bigr),
\]
where default negation is interpreted using strong negation $I({not\,}c)=1-I(c)$.

\section{Decision--Propagation for ASP}
In this section, we describe \modelname computations, and show that successful \modelname computations capture stable model semantics.
\modelname uses a forward chaining approach similar to~\cite{DBLP:journals/tplp/LefevreBSG17,DBLP:journals/fuin/PaluDPR09}, with the biggest difference being that the decisions are made at an atom and not at a rule level, making the differentiable extension much simpler.

\begin{figure}[t]
    \centering
    \includegraphics[width=\linewidth]{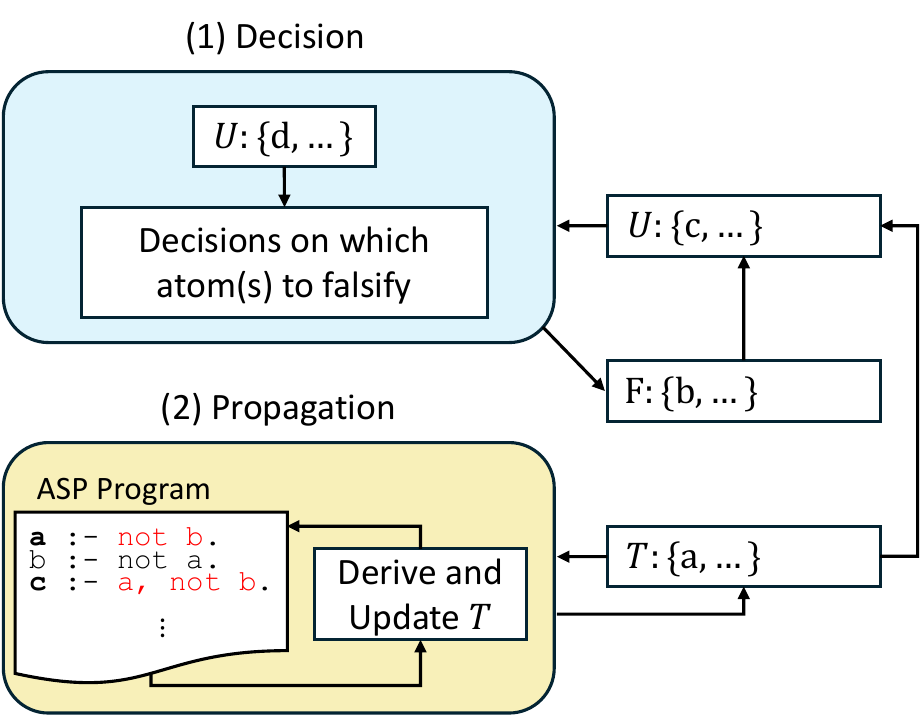}
    \caption{Decision--Propagation (\modelname) iteration pipeline.}
    \label{fig:ndprop_arch}
\end{figure}

Formally, \modelname operates on an atom set $\mathcal{A}$ of $P$,
with each state splitting $\mathcal{A}$ into disjoint true, false, and undecided sets $\Tr$, $\Fa$, and $\Un=\mathcal{A}\setminus(\Tr\cup \Fa)$, respectively. 
With initial sets $\Tr^0=\Cpnorm{P}{\mathcal{A}}{\infty}{\emptyset},\Fa^0=\emptyset$, and $\Un^0=\mathcal{A}\setminus(\Tr^0\cup \Fa^0)$, a \emph{run} in \modelname performs the following iterations $(k\ge 1)$ until termination:
\begin{enumerate}
  \item \textbf{Decision.} Select an undecided atom $a\in \Un^{k-1}$ (via some policy), then update $\Fa^{k}= \Fa^{k-1}\cup\{a\}$.
  \item \textbf{Propagation.} With $\Fa^{k}$ fixed, compute $\Tr^k = C^\infty_{P,\mathcal{A}\setminus\Fa^k}(\Tr^{k-1})$ and set $\Un^{k}= \Un^{k-1}\setminus(\Tr^k\cup\Fa^k)$.
  \item \textbf{Termination.} If $\Un^k=\emptyset$, stop and return $\Tr^k$; otherwise continue to the next decision.
\end{enumerate}
When a run terminates, we write $k=\infty$, obtaining $\Tr^\infty$, $\Fa^\infty$, and $\Un^\infty$.  
Here, both $\Tr$ and $\Fa$ increase monotonically because Propagation
and Decision never remove atoms from them.

\begin{lemma}
\label{lem:termination}
    A \modelname run always terminates for a finite atom set $\mathcal{A}$.
\end{lemma}
\begin{proof}
    Decisions always remove an atom $a$ from $\Un^k$ at every iteration. Thus, the number of \modelname iterations is finite, if the atom set $\mathcal{A}$ is finite.
\end{proof}

\begin{definition}[Successful \modelname computation]
\label{def:contr-free}
A successful \modelname\ computation is a \modelname run that terminates with $\Un^\infty=\emptyset$ and $\Tr^\infty\cap \Fa^\infty=\emptyset$.
\end{definition}

The following properties of the operator $C_{P,I}(J)$ are not hard to see.
\begin{lemma}
\label{lem:PF-monotone}
If $\Fa\subseteq {\Fa'}$, then
$\Cp{P}{\Fa}{\infty}{\emptyset} \subseteq \Cp{P}{\Fa'}{\infty}{\emptyset}$.
\end{lemma}

\begin{proof}
\edit{Let $J=\mathcal{A}\setminus\Fa$ and $J'=\mathcal{A}\setminus\Fa'$.
Since $\Fa\subseteq\Fa'$, we have $J'\subseteq J$.
If a rule satisfies $B^-(r)\cap J=\emptyset$, then also
$B^-(r)\cap J'=\emptyset$. As the reducts $P^J$ and $P^{J'}$ are obtained by removing rules whose negative body intersects $J$ and $J'$ respectively, and removing negative literals from the remaining rules, $P^J\subseteq P^{J'}$. By standard properties of definite programs, their least models are monotone with respect to program inclusion, yielding $\mathcal{T}_{P^{J}}^\infty(\emptyset)=\Cp{P}{\Fa}{\infty}{\emptyset} \subseteq \mathcal{T}_{P^{J'}}^\infty(\emptyset)=\Cp{P}{\Fa'}{\infty}{\emptyset}$.}
\end{proof}

\begin{lemma}
\label{lemma:least-model-induced}
If $\Tr^{k-1}\subseteq \Cp{P}{\Fa^k}{\infty}{\emptyset}$,
then $\Cp{P}{\Fa^k}{\infty}{\Tr^{k-1}}=\Cp{P}{\Fa^k}{\infty}{\emptyset}$.
\end{lemma}

\begin{proof}
\edit{By assumption, $\Tr^{k-1}\subseteq\Cp{P}{\Fa^k}{\infty}{\emptyset}$. Since $C_{P,\mathcal{A}\setminus\Fa^k}$ is monotone and $\Cp{P}{\Fa^k}{\infty}{\emptyset}$ is its fixpoint, for every $n\ge 0$ we have $\Cp{P}{\Fa^k}{n}{\emptyset}\subseteq\Cp{P}{\Fa^k}{n}{\Tr^{k-1}}\subseteq\Cp{P}{\Fa^k}{\infty}{\emptyset}$. Since $\mathcal A$ is finite, there exists $m$ such that $\Cp{P}{\Fa^k}{m}{\emptyset}=\Cp{P}{\Fa^k}{\infty}{\emptyset}$. Hence, $\Cp{P}{\Fa^k}{\infty}{\emptyset}\subseteq\Cp{P}{\Fa^k}{m}{\Tr^{k-1}}\subseteq\Cp{P}{\Fa^k}{\infty}{\emptyset}$, so $\Cp{P}{\Fa^k}{m}{\Tr^{k-1}}=\Cp{P}{\Fa^k}{\infty}{\emptyset}$. Therefore, $\Cp{P}{\Fa^k}{\infty}{\Tr^{k-1}}=\Cp{P}{\Fa^k}{\infty}{\emptyset}$.}
\end{proof}

\begin{theorem}
\label{thm:termination-stable}
Let $P$ be a finite program and $S \subseteq \mathcal{A}_P$, where $\mathcal{A}_P$ is the set of atoms occurring in $P$.
Then $S$ is a stable model of $P$ if and only if $S$ is the final true set of a successful \modelname computation.
\end{theorem}

\begin{proof}
\textbf{($\Rightarrow$)} Let $S$ be a stable model of $P$, and $\Fa^{k-1}$, $\Tr^{k-1}$ and $\Un^{k-1}$ be the sets of atoms that are true, false and undefined at the beginning of the $k$-th iteration, where $k=\infty$ is the final iteration.
Consider a run of \modelname that at each step decides the falsity of an atom $a \in (\mathcal{A}\setminus S)\cap\Un^{k-1}$. The propagation step then yields $\Tr^k=\Cpp{P}{\Fa^{k}}{\infty}{\Tr^{k-1}}$, which by Lemma~\ref{lemma:least-model-induced} equals $\Cpp{P}{\Fa^{k}}{\infty}{\emptyset}$.
Because $\Fa^k\subseteq\mathcal{A}_P\setminus S$, we have $S\subseteq\mathcal{A}_P\setminus \Fa^k$ and hence, by Lemma~\ref{lem:PF-monotone},
$\Cpp{P}{\Fa^{k}}{\infty}{\emptyset}\subseteq C_{P,S}^{\infty}(\emptyset)$. By Lemma~\ref{lem:cp_sm}, $C_{P,S}^{\infty}(\emptyset)=S$,
so $\Tr^k\subseteq S$ and $\Tr^k\cap\Fa^k=\emptyset$ for all $k$.
At termination, $\Un^\infty=\emptyset$, so $\Tr^\infty=\Cpp{P}{\Fa^\infty}{\infty}{\emptyset}=\Cpnorm{P}{\Tr^\infty}{\infty}{\emptyset}$ hence $\Tr^\infty=S$ from Lemma~\ref{lem:cp_sm}.
From Lemma~\ref{lem:termination}, the computation will terminate as the given program $P$ is finite.

\textbf{($\Leftarrow$)} Let $S$ be the final true set of a successful \modelname computation, and $\Tr^{\infty-1}$ and $\Fa^{\infty}$ be the true and false sets before the final propagation step. As the propagation step does not modify $\Fa^{\infty}$, from Definition~\ref{def:contr-free}, $S=\mathcal{A}\setminus \Fa^{\infty}$ so $S=\Cpnorm{P}{\mathcal{A}\setminus\Fa^\infty}{\infty}{\Tr^{\infty-1}}=\Cpnorm{P}{S}{\infty}{\Tr^{\infty-1}}$.
From Lemma~\ref{lemma:least-model-induced}, $\Cpnorm{P}{S}{\infty}{\Tr^{\infty-1}}=\Cpnorm{P}{S}{\infty}{\emptyset}$, and Lemma~\ref{lem:cp_sm} implies that $S$ is a stable model of $P$.
\end{proof}

\begin{remark}[Batch decisions]
\label{rem:batch-decisions}
The proof only relies on the fact that $\Fa$ grows monotonically and is held fixed during propagation. Hence any variant that, in a single decision step, adds a set $E \subseteq \Un$ to $\Fa$ (i.e., $\Fa \leftarrow \Fa \cup E$) satisfies the same invariants and therefore admits the same characterization: every \modelname\ computation of such a variant yields a stable model.
\end{remark}

\begin{figure*}[t]
    \centering
    \includegraphics[width=0.9\linewidth]{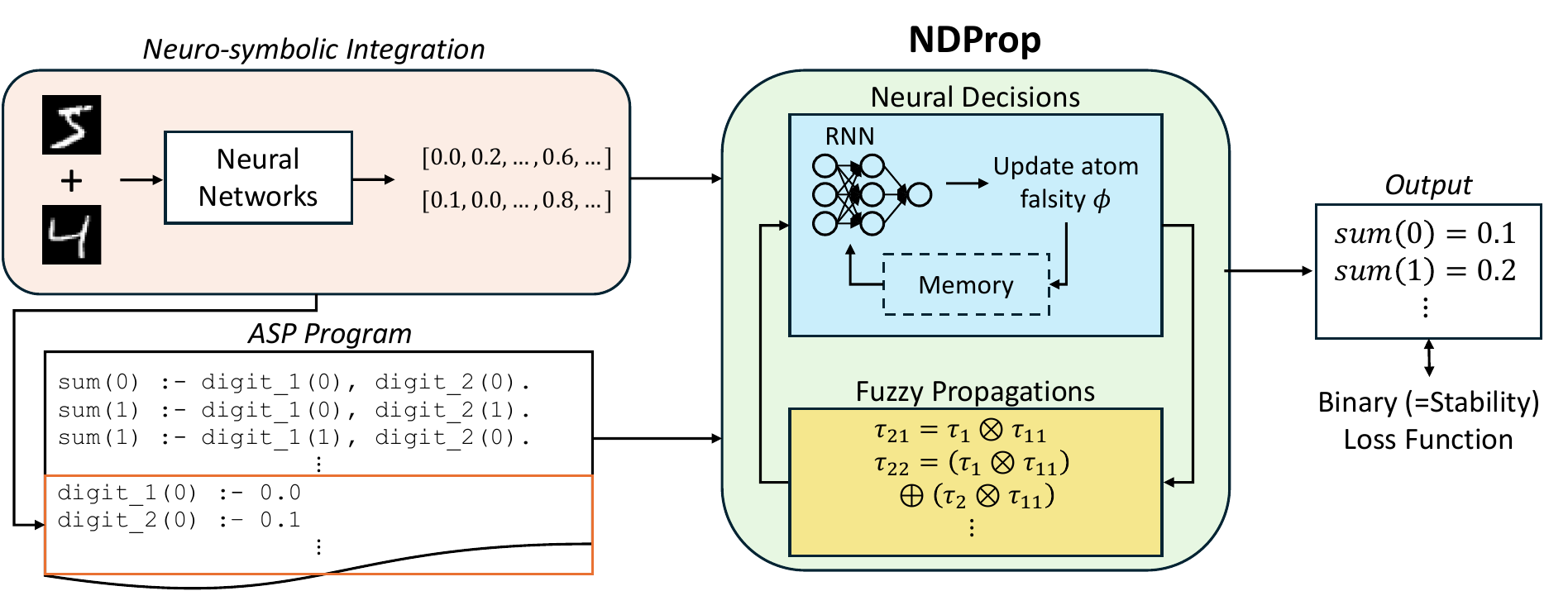}
    \caption{\neuralmodelname pipeline. The neuro-symbolic integration is added for specific instantiations, and the arrow leading from the neuro-symbolic block to the ASP program is added only when the backbone neural networks are being trained.}
    \label{fig:architecture}
\end{figure*}

\section{Neural Decision--Propagation}
We now instantiate \modelname\ with differentiable operators, yielding a neural solver \neuralmodelname\ that is an extension of \modelname. \edit{The overview of the \neuralmodelname\ pipeline is shown in the middle section of Figure~\ref{fig:architecture}.}
Corresponding to the true and false sets $\Tr$ and $\Fa$, we consider true and false \emph{degrees} $(\tau,\phi)\in[0,1]^n\times[0,1]^n$.
Furthermore, decisions are implemented via neural decision modules that gradually increase falsity.

\subsection{Neural Extensions of \modelname}
Each atom $a_i$ carries a true and false degree $\tau_i$ and $\phi_i$.
Literal evaluations are done with the $\mathrm{val}$ operator, which outputs for positive and \edit{negative} literals
\[
  \mathrm{val}(a_i)=\tau_i \ \text{ and }\ \mathrm{val}({not}\,a_i)=\phi_i,
\]
respectively. Regarding $(\tau_i, \phi_i)$,  values near $(0,0)$ indicate undecided atoms, $(1,0)$ denotes derived truth, $(0,1)$ denotes explicit falsity, and $(1,1)$ encodes contradiction.
For \neuralmodelname, we combine these into a single membership score
\[
p_i=\tfrac12(\tau_i + 1 - \phi_i)\in[0,1],
\]
which serves as the atom’s soft truth value indicating whether the atom should be included in the answer set or not.
Importantly, this becomes binary iff $(\tau_i, \phi_i)$ is $(1,0)$ or $(0,1)$, meaning that if the variable is either undecided $(0,0)$ or $(1,1)$ contradictory, the score will never become binary.
We denote by $(\tau^{t},\phi^{t})$ the degrees at iteration~$t$.

\subsubsection{Decision Phase}
Decisions update only $\phi$ while holding $\tau$ fixed, as in \modelname.
Given a t-norm $\otimes$ and t-conorm $\oplus$, the undecided degree $\mu_i$ for atom $a_i$ is defined as
\[
  \mu^t_i := 1-(\tau_i^t\oplus\phi_i^t) = (1-\tau^{t}_i) \otimes (1-\phi^{t}_i) \in [0,1].
\]
We then produce decisions via a neural decision module, which in this case is a Recurrent Neural Network (RNN) that takes the current input and hidden state $s^{t-1}$ and yields the decision logits $d^t$ and updated hidden state $s^t$ as follows:
\[
  (d^t,s^t) = \mathrm{RNN}_{\mathrm{dec}}\!\left([\tau^t\,\|\,\phi^t\,\|\,\mu^t],\,s^{t-1}\right),
\]
where $\|$ denotes the concatenation operation.
RNNs retain historical information with the hidden states $s^t$, making it suitable for our approach that requires producing decisions over multiple steps.
\edit{We then apply a sigmoid function $\sigma: \mathbb{R}\to(0,1)$ on the decision logits to yield the decision scores:}
\[
  \delta^t = \sigma\!\left(d^t\right).
\]
The falsity degrees move toward these selections via
\begin{equation}\label{eq:phi}
\phi^{t+1} = \phi^{t} + (\mu^t \odot \delta^t),
\end{equation}
with $\odot$ denoting elementwise multiplication.
\edit{Note that $\phi^{t+1}$ remains in $[0,1]$, since $\mu^t\leq1-\phi^t$ and $\delta^t\in(0,1)$.}
Because $\mu^t$
becomes zero for already decided atoms, this additive rule directly mirrors the discrete falsity decision ($\Fa^{t}= \Fa^{t-1}\cup E$ cf. Remark~\ref{rem:batch-decisions}) when $(\tau^t,\phi^t)$ and $\delta^t$ are binary.

\subsubsection{Propagation Phase}
Propagation updates only $\tau$ while holding $\phi$ fixed, as in \modelname. Given a t-norm $\otimes$ and t-conorm $\oplus$, for each rule $r_j$, its rule body contributes support
\[
  \rho_j(\tau,\phi) := \bigotimes_{\ell\in B(r_j)}\mathrm{val}(\ell).
\]
Each head atom $a_i$ aggregates these supports via the soft immediate consequence operator
\[
  (\FCp{P}{\phi}{\tau})_i := \bigoplus_{j:H(r_j)=a_i} \rho_j(\tau,\phi).
\]
Given $\phi^{t+1}$, we update $\tau$ from the current state $\tau^{t}$ by iterating $\mathcal{C}_{P,\phi^{t+1}}$ until convergence:
\begin{equation*}
    \tau^{t+1}=\mathcal{C}_{P,\phi^{t+1}}^{\infty}({\tau^{t}}).
\end{equation*}
In practice, we stop the iteration early when the maximum change between successive applications of $\mathcal{C}_{P,\phi^{t+1}}$ falls below a fixed convergence threshold, or the iteration reaches a given pre-defined value.

\subsubsection{Crisp Case Recovers \modelname}
We give the proof that \modelname can be emulated with a specific constraint on \neuralmodelname, showing that \neuralmodelname is a differentiable extension of \modelname.

\begin{definition}[Set induced by $\psi$]
Let $\psi\in[0,1]^n$ be a real valued vector with the $i$-th value $\psi_i$ corresponding to atom $a_i$. The set induced by $\psi$ is defined as
\[
[\psi] := \{\, a_i \mid \psi_i = 1 , 1\leq i \leq n\,\}.
\]
\end{definition}

In the crisp case, the operator $\mathcal C_{P,\phi}$ makes the same atoms true as $C_{P,\mathcal{A}\setminus \Fa}$, that is:

\begin{lemma}
\label{lem:cp_crisp}
    Let $(\tau,\phi)\in\{0,1\}^{n}\times\{0,1\}^{n}$, $\Tr=[\tau]$ and $\Fa=[\phi]$. Then, $[\FCp{P}{\phi}{\tau}]=\Cp{P}{\Fa}{}{\Tr}$.
\end{lemma}

\begin{proof}
\edit{With binary $(\tau,\phi)$, a positive literal evaluates to $\tau_i\in\{0,1\}$ and a default-negated literal evaluates to $\phi_i\in\{0,1\}$ for atom $a_i$.
Thus each rule $r_j$'s support $\rho_j(\tau,\phi)$ is $1$ exactly when $B^+(r_j)\subseteq [\tau]$ and $B^-(r_j)\subseteq [\phi]$.
Aggregating supports via t-conorm $\oplus$ yields exactly the heads derived by $C_{P,\mathcal{A}\setminus\Fa}(\Tr)$, so the induced true set coincides.}
\end{proof}

Then we obtain the following property:
\begin{proposition}[Crisp Case]
\label{prop:binary-limit}
Let $(\tau,\phi)\in\{0,1\}^{n}\times\{0,1\}^{n}$, $\delta\in\{0,1\}^n$, and propagation compute the fixpoint of $\mathcal C_{P,\phi}$.
Then, one update of \neuralmodelname\ coincides with one iteration of \modelname.
\end{proposition}

\begin{proof}
With binary $(\tau,\phi)$, the set induced by undecided degree $\mu=(1-\tau)\otimes(1-\phi)$ corresponds to $\Un=\mathcal A\setminus(\Tr\cup\Fa)$ when $\Tr=[\tau]$ and $\Fa=[\phi]$.
Thus, the update $\phi'=\phi+(\mu\odot\delta)$ coincides with $\Fa'=\Fa \cup (\Un \cap [\delta])$. 
From Remark~\ref{rem:batch-decisions}, deciding falsity of a set of atoms is fine, thus this exactly coincides with the decision step of \modelname when the falsity of a set of atoms $\Un \cap [\delta]$ is decided.

Since $(\tau,\phi')\in\{0,1\}^{n}\times\{0,1\}^{n}$, from Lemma~\ref{lem:cp_crisp} $[\mathcal C_{P,\phi'}(\tau)]=C_{P,\mathcal{A}\setminus \Fa'}(\Tr)$.
When $\mathcal C_{P,\phi'}(\tau)$ is applied until a fixpoint, this exactly coincides with the propagation step of \modelname.
Furthermore, when $\mu,\delta$ is binary, $\tau$ and $\phi$ are guaranteed to be binary as well.
\end{proof}

\subsection{Loss Alignment}

We prove that
enforcing \neuralmodelname to output binary values is equivalent to enforcing stability. 
Here, $\psi' \le \psi$ is equivalent to $\forall i\;(\psi'_i \le \psi_i)$ and $\textbf{0}$ and $\textbf{1}$ are $n$-dimensional vectors of all $0$s and all $1$s, respectively.

\begin{lemma}
\label{lem:false_monotonic}
Let $\phi\le\phi'$. 
Then, 
$\mathcal{C}_{P,\phi}(\tau)\le \mathcal{C}_{P,\phi'}(\tau)$, and consequently
$\mathcal{C}_{P,\phi}^\infty(\textbf{0})\le \mathcal{C}_{P,\phi'}^\infty(\textbf{0})$.
\end{lemma}

\begin{proof}
\edit{For fixed $\tau$, negative literals evaluate as $\mathrm{val}({not}\, a_i)=\phi_i$.
Since the t-norm $\otimes$ and t-conorm $\oplus$ are monotone in each argument,
$\phi\le\phi'$ implies each rule support $\rho(\tau,\phi)\le\rho(\tau,\phi')$,
hence $\mathcal{C}_{P,\phi}(\tau)\le \mathcal{C}_{P,\phi'}(\tau)$.
Because $\mathcal{C}_{P,\phi}\le \mathcal{C}_{P,\phi'}$ pointwise, iterating from $\textbf{0}$ yields
$\mathcal{C}_{P,\phi}^k(\textbf{0})\le \mathcal{C}_{P,\phi'}^k(\textbf{0})$ for all $k$, and taking the limit gives
$C_{P,\phi}^\infty(\textbf{0})\le C_{P,\phi'}^\infty(\textbf{0})$.}
\end{proof}

\begin{lemma}
    \label{lem:monotonic}
    Let $\tau^{t-1}\leq \mathcal{C}_{P,\phi^{t}}^\infty(\textbf{0})$. Then, $\mathcal{C}_{P,\phi^{t}}^\infty(\tau^{t-1})=\mathcal{C}_{P,\phi^{t}}^\infty(\textbf{0})$.
\end{lemma}

\begin{proof}
\edit{The operator $\mathcal{C}_{P,\phi^t}$ is monotone in $\tau$ by monotonicity of $\otimes$ and $\oplus$. Since $\tau^{t-1}\le\mathcal{C}_{P,\phi^{t}}^\infty(\textbf{0})$ and $\mathcal{C}_{P,\phi^{t}}^\infty(\textbf{0})$ is a fixpoint of $\mathcal{C}_{P,\phi^t}$, for every $n\ge0$ we have $\mathcal{C}_{P,\phi^t}^n(\textbf{0})\le\mathcal{C}_{P,\phi^t}^n(\tau^{t-1})\le\mathcal{C}_{P,\phi^{t}}^\infty(\textbf{0})$. Taking limits yields $\mathcal{C}_{P,\phi^{t}}^\infty(\tau^{t-1})=\mathcal{C}_{P,\phi^{t}}^\infty(\textbf{0})$.}
\end{proof}

\begin{theorem}[Binary terminal state yields stability]
\label{thm:binary-terminal-stability}
Let $P$ be a finite program and $S \subseteq \mathcal{A}_P$, where $\mathcal{A}_P$ is the set of atoms occurring in $P$.
Then $S$ is a stable model of $P$ if and only if propagation is applied until convergence with the resulting membership score $p_i\in\{0,1\}$ and $S=[p]$.
\end{theorem}

\begin{proof}
($\Leftarrow$)
Let $p\in\{0,1\}^n$ be the membership score obtained from $(\tau^\infty,\phi^\infty)$, and let $\Tr=[\tau]$ and $\Fa=[\phi]$.
Since $p_i=\tfrac12(\tau_i+1-\phi_i)$ and $\tau^\infty_i,\phi^\infty_i\in[0,1]$, $p_i\in\{0,1\}$ implies $(\tau^\infty_i,\phi^\infty_i)\in\{(1,0),(0,1)\}$, and thus $\Tr\cap\Fa=\emptyset$, $\Tr\cup\Fa=\mathcal{A}$ and $\Tr=[p]$.
$\phi$ increases monotonically as there are no negatives in Equation~(\ref{eq:phi}).
As propagation is applied until convergence
$\tau^\infty=\mathcal{C}_{P,\phi^\infty}^\infty(\tau^{\infty-1})$, and  from Lemmas~\ref{lem:false_monotonic} and \ref{lem:monotonic}, $\tau^\infty=\mathcal{C}_{P,\phi^\infty}^\infty(\textbf{0})$.
Since $[\textbf{0}]=\emptyset$, from Lemma~\ref{lem:cp_crisp}, $[\tau^\infty]=[\mathcal{C}_{P,\phi^\infty}^\infty(\textbf{0})]=C_{P,\mathcal{A}\setminus\Fa}^\infty(\emptyset)=C_{P,T}^\infty(\emptyset)$. 
From Lemma~\ref{lem:cp_sm}, $[p]=[\tau^\infty]$ is thus a stable model of $P$.

($\Rightarrow$)
Let $\tau^\infty$ and $\phi^\infty$ be the true and false degrees after the final iteration.
From Proposition~\ref{prop:binary-limit}, if we set the decision scores to be binary
and make the propagation phase compute until a fixpoint of $\mathcal C_{P,\phi}$, the iteration coincides with a \modelname iteration. From Theorem~\ref{thm:termination-stable}, there exists a successful \modelname computation that ends with the stable model $S$, and since $\tau,\phi$ is binary, $\Tr\cap\Fa=[\tau^\infty]\cap[\phi^\infty]=\emptyset$ and $\Un=\emptyset$, $(\tau^\infty,\phi^\infty)\in\{(\textbf{1},\textbf{0}),(\textbf{0},\textbf{1})\}$ and thus $p$ is binary and $[p]=S$.
\end{proof}

From Theorem~\ref{thm:binary-terminal-stability}, we can conclude that when the output $p$ of \neuralmodelname is trained with a loss function that enforces binary outputs, it inherently forces the \neuralmodelname to learn to output stable models, given that the final propagation is done until convergence. 
This holds even with a single iteration, since the decision step can falsify a set of atoms at once, c.f. Remark~\ref{rem:batch-decisions}.

\section{Experiments}
In this section, we perform experiments to answer the following three questions: 
\smallskip
\begin{itemize}[leftmargin=*,itemsep=0.5pt,topsep=0.5pt]
    \item[] (\textbf{Q1}) Can \neuralmodelname learn to solve ASP problems? 
    \item[] (\textbf{Q2}) Can \neuralmodelname compete with existing ASP-based neuro-symbolic frameworks?
    \item[] (\textbf{Q3}) What are novel settings 
    in explicit need of \neuralmodelname?
\end{itemize}
\smallskip

\begin{table*}[t]
\centering
\small
\setlength{\tabcolsep}{3pt}
\renewcommand{\arraystretch}{1.15}
\begin{tabular}{lcccccccccccccccccc}
\toprule
& \multicolumn{3}{c}{Addition}
& \multicolumn{3}{c}{Addition2}
& \multicolumn{3}{c}{Add2x2}
& \multicolumn{3}{c}{Apply2x2}
& \multicolumn{3}{c}{Member3}
& \multicolumn{3}{c}{Member5} \\
\cmidrule(lr){2-4} \cmidrule(lr){5-7} \cmidrule(lr){8-10} \cmidrule(lr){11-13} \cmidrule(lr){14-16} \cmidrule(lr){17-19}
& {\tiny \neuralmodelname} & {\tiny NASP} & {\tiny SLASH}
& {\tiny \neuralmodelname} & {\tiny NASP} & {\tiny SLASH}
& {\tiny \neuralmodelname} & {\tiny NASP} & {\tiny SLASH}
& {\tiny \neuralmodelname} & {\tiny NASP} & {\tiny SLASH}
& {\tiny \neuralmodelname} & {\tiny NASP} & {\tiny SLASH}
& {\tiny \neuralmodelname} & {\tiny NASP} & {\tiny SLASH} \\
\midrule
Train & \textbf{160} & {1612} & 430 & \textbf{235} & TO & 2374 & \textbf{692} & TO & TO & {346} & {385} & \textbf{166} & \textbf{110} & 2829 & TO & \textbf{969} & TO & TO \\
Test & \textbf{0.37} & 24.25 & 9.94 & \textbf{1.28} & 33.90 & 3.73 & \textbf{1.15} & 41.59 & 3.75 & \textbf{1.52} & 5.31 & 4.32 & \textbf{0.79} & 5.13 & 3.08 & \textbf{4.40} & 496.4 & 6.38  \\
\hdashline
\addlinespace[0.5ex]
R.Acc & \textbf{99.02} & 98.54 & 98.91 & 98.92 & 98.46 & \textbf{99.00} & \textbf{99.04} & 98.59 & 98.94 & \textbf{100.0} & \textbf{100.0} & \textbf{100.0} & \textbf{97.25} & 96.81 & 97.16 & \textbf{96.97} & 11.56 & 8.92 \\
I.Acc & \textbf{98.20} & 97.11 & 97.83 & 96.10 & 94.57 & \textbf{96.47} & \textbf{95.73} & 93.70 & 95.17 & \textbf{100.0} & \textbf{100.0} & \textbf{100.0} & \textbf{98.97} & 98.57 & 98.73 & \textbf{97.53} & 59.67 &  60.30\\
\bottomrule
\end{tabular}
\caption{\edit{Neuro-symbolic MNIST results for \neuralmodelname, NeurASP (NASP) and SLASH, all trained for 10 epochs.
Metrics are Train: Training time, Test: Testing time, R.Acc: Recognition Accuracy (\%), and I.Acc: Inference accuracy (\%).}}
\label{tab:mnist-results}
\end{table*}

\begin{table}[t]
\centering
\small
\begin{tabular}{llcccccc}
\toprule
\multirow{2}{*}{{Method}} & \multirow{2}{*}{{Split}} & \multicolumn{3}{c}{{N2L}} & \multicolumn{3}{c}{{3-LP}} \\
\cmidrule(lr){3-5} \cmidrule(lr){6-8}
 &  & E & M & H & E & M & H \\
\midrule
\multirow{2}{*}{{\neuralmodelname}} & E & \textbf{100} & \underline{80} & \underline{17} & \textbf{100} & \underline{61} & \underline{12} \\
 & M & \textbf{100} & \textbf{96} & \textbf{42} & \textbf{100} & \textbf{86} & \textbf{37} \\
\hdashline
\addlinespace[0.5ex]
\multirow{1}{*}{{Random}} & -- & 77 & 1 & 0 & 56 & 0 & 0 \\
\multirow{1}{*}{{R\modelname-1}} & -- & 40 & 9 & 0 & 7 & 0 & 0 \\
\multirow{1}{*}{{R\modelname-10}} & -- & 96 & 37 & 4 & 60 & 3 & 0 \\
\multirow{1}{*}{{R\modelname-100}} & -- & \textbf{100} & 68 & 15 & {99} & 14 & 0 \\
\bottomrule
\end{tabular}
\caption{Percentage of instances for which a stable model was found, when trained on the easy and medium sets of each corresponding dataset. R\modelname-$K$ corresponds to \modelname with random decisions, with a total of $K$ restarts. Best accuracy is in bold, and second best accuracy is underlined.}
\label{tab:n2l-results}
\end{table}

\subsection{Random Logic Programs}
\edit{To assess \neuralmodelname's behavior on unstructured problems, we first evaluate it on randomly generated logic programs. Random instances are utilized to evaluate neural approaches to symbolic solving, e.g., in Boolean Satisfiability~\cite{DBLP:conf/iclr/SelsamLBLMD19}, where they are used to evaluate whether learned models can learn to solve problems that are generally difficult to find heuristics for.
}

We follow the linear model $L(N2,c_1,c_2)$ for random negative two-literal (N2L) programs~\cite{DBLP:journals/tplp/WangWM15}. For $n$ atoms, each pure rule $a\leftarrow{not \,}b$ (with $a\neq b$) appears independently with probability $c_1/n$ and each contradiction rule $a\leftarrow{not \,}a$ appears with probability $c_2/n$.
Random N2L programs are attractive for an experimental setting as they are easy to produce, while being considerably difficult, as 
it includes NP-complete instances. The same setting has been used in 
other work that attempted to apply GNNs to Answer Set Computation~\cite{DBLP:conf/cilc/IeloR21}.

As supported models of N2L programs are known to be stable models as well, we look at problems that are not in this regard. Specifically, we look at the $k$-$LP(N,L)$ generator that has been studied for phase transitions in ASP~\cite{DBLP:conf/iclp/ZhaoL03}. We use the 3-LP generator, which has been shown to be hard around $L/N=5$, where $L$ and $N$ correspond to the number of rules and atoms respectively.

\subsubsection{Setup}
We focus on having three datasets per generator (N2L, and 3-LP) with different number $n$ of atoms. Specifically, we use easy (5--10), medium (11--50), and hard (51-100). We specifically use the hard set to determine whether the model is able to generalize to bigger problems.
Each dataset consists of 1000 training and 100 validating instances and 100 testing instances.
As the baseline, we implement R\modelname, which is \modelname with random decisions.
We test R\modelname with 1, 10, and 100 restarts, denoted as R\modelname-1, R\modelname-10, and R\modelname-100 respectively.
We compare \neuralmodelname with R\modelname to see whether \neuralmodelname is able to learn the decision making process to outperform random decisions.

For \neuralmodelname, we train for 1000 epochs, with the setup of an iteration of 10 (50 at test time), a hidden dimension of 32, a logical dimension of 32, a learning rate of $1\times10^{-3}$, and the G\"odel t-norm. The logical dimension means the model maintains multiple logical states at once, and we select the state whose output is closest to binary at the end. 
For the loss function, we use the smallest Binary Cross Entropy (BCE) loss over all stable models as the objective function.

\subsubsection{Results}
Table~\ref{tab:n2l-results} provides the percentage of problems in each test set solved.
The results show that \neuralmodelname is effectively learning the decision dynamics for solving random problems.
Specifically, \neuralmodelname is shown to completely outperform R\modelname in the 1-shot setting, with performance on medium and easy training datasets exceeding that of 100-shot R\modelname.
This empirically supports our claim that \neuralmodelname is an extension of \modelname, and that it can learn the decision heuristics for finding stable models for logic programs.

\subsection{Neuro-Symbolic Tasks}
As \neuralmodelname can approximate answer sets differentiably, it is natural to combine with existing neural networks based architectures to perform neuro-symbolic inference.
In all settings, the input consists of one or more MNIST images, and supervision is provided only through arithmetic relations between the underlying concepts, rather than through per-image labels~\cite{DBLP:journals/ai/ManhaeveDKDR21,DBLP:conf/icml/GauntBKT17,DBLP:conf/aaai/TsamouraHM21}. The goal is therefore to jointly learn perception using symbolic reasoning from weak, relational supervision.
\edit{This will allow us to evaluate the merits of using an end-to-end differentiable system, as opposed to a system that relies on separate training of the perception module and the reasoning module.}

\paragraph{ASP Encodings.}
In contrast to prior neuro-symbolic ASP systems that use neural predicates to interface with perception modules~\cite{DBLP:conf/ijcai/YangIL20}, we adopt a different strategy. We embed the output of neural models as fuzzy facts into a fixed ASP program that encodes the arithmetic constraints \edit{(Figure~\ref{fig:architecture})}. While this may seem to force \neuralmodelname to perform as a FASP solver, the objective is to produce binary outputs, explicitly forcing the backbone perception model to predict binary values as well.

\subsubsection{Setup}
In this experiment, we perform evaluations on six variants of the MNIST-based arithmetic reasoning tasks that are commonly used as neuro-symbolic benchmarks~\cite{DBLP:conf/ijcai/YangIL20,DBLP:journals/ai/ManhaeveDKDR21}: 1 and 2 digit addition, Add2x2, Apply2x2 (with handwritten operator images) and 3 and 5 digit membership. 
For \neuralmodelname, we use the setup with iteration of 1, hidden dimension of 8, logical dimension of 2, learning rate of $1\times10^{-3}$, and product t-norm.

We compare the performance of \neuralmodelname with ASP-based neuro-symbolic systems: NeurASP~\cite{DBLP:conf/ijcai/YangIL20} and SLASH~\cite{DBLP:journals/jair/SkryaginODK23}. Experiments are conducted with a maximum of 10 epochs for all approaches, with a timeout set to 3000 seconds; the training is cut off when the budget is reached, and accuracy is checked at the end of the epoch. The training setup used is identical to those given in the original papers.

\subsubsection{Results}
Table~\ref{tab:mnist-results} reports recognition accuracy, inference accuracy (accuracy of final output), and runtime (seconds) for both training and testing splits, averaged over 3 runs. \edit{Across six neuro-symbolic tasks, \neuralmodelname\ achieves higher digit and inference accuracy than NeurASP and SLASH for five, with the exception of the Addition2 task, where SLASH has a slightly higher recognition and inference accuracy.
Empirically, this shows the benefits of using end-to-end trainable systems: jointly optimizing perception and reasoning leads to more cohesive and consistent outputs, leading to performance similar to or better than systems that rely on symbolic solvers. This is similar to findings in prior works that use differentiable computation of supported models~\cite{DBLP:conf/ecai/TakemuraI24}, yielding faster training time and higher accuracy in certain setups.}

Furthermore, the training time for \neuralmodelname is significantly faster than NeurASP and SLASH. 
Specifically, for the Member5 task, both NeurASP and SLASH fail to train even one epoch within the time budget, leading to very low accuracy.
On the other hand, \neuralmodelname has a longer training time for Apply2x2; this is thought to be because there are extremely many grounded atoms in this task, 
which have to be processed at once for the forward pass of \neuralmodelname, leading to a much higher computational cost.  

{The difference is more prominent in the testing speed, where \neuralmodelname is an average of 41.8 times faster than NeurASP, and 6.8 times faster than SLASH.
This is mainly due to the GPU capabilities of \neuralmodelname; GPUs allow for higher parallelization, allowing our model to process multiple instances with one forward pass, even during testing.
In comparison, NeurASP can only process one instance at a time, while SLASH can process multiple instances only at training time.}

\subsection{Reasoning with Incomplete Data}
In this experiment, we look at settings where data is incomplete.
Given the predicted outputs of models with incomplete training, \neuralmodelname can add a layer of reasoning on top,
yielding a solution that is both correct in terms of the given requirements and is able to account for the information given by the incomplete data.

\subsubsection{Setup}
Visual sudoku is a task where each cell has an image corresponding to a number, and the goal is to find the correct solution for the given grid, with MNIST being the most common image used.
In this experiment, we hypothesize a setting where the neural network is fixed; we pretrain the model on MNIST images, and reason on its predictions to find the solution to the task.
By having models that are trained on less data, we can simulate a situation in which the given predictions are not correct 100\% of the time.

We specifically focus on the 4x4 setting with each cell containing an image with a number between 0-4 (0 for empty cell). We vary the amount of labeled training data for the MNIST predictor (100\%, 50\%, 10\%, 5\% and 1\%) with 500 steps of training, to test whether the framework can perform under imperfect settings. We compare against two baselines that use the same predictions:
\begin{itemize}
  \item \textbf{ASP:} Give the current predictions other than 0 as a fact, and solve the ASP program.
  \item \textbf{Neural:} Train a neural classifier that takes as input the predicted values of each cell (0 for empty cells), and predict the digit in each cell. This is trained until convergence.
\end{itemize}
In this setting, our framework is used in a separate way to the neuro-symbolic tasks. As the input facts are not trained to become binary, we use the facts only as a heuristic to guide \neuralmodelname to a solution. This is provided through the RNN's initial hidden states $s^0$ for each atom, which is computed by a single-layer MLP.
We train for 100 epochs \edit{(which is sufficient for the model to converge)} using \neuralmodelname with iterations of 3, hidden dimension of 8, logical dimension of 2, learning rate of $1\times10^{-3}$, and G\"odel t-norm.

\begin{table}[t]
\centering
\small
\setlength{\tabcolsep}{6pt}
\renewcommand{\arraystretch}{1.15}
\begin{tabular}{lccccc}
\toprule
 & {100\%} & {50\%} & {10\%} & {5\%} & {1\%} \\
\midrule
{\neuralmodelname} & \textbf{95.33} & \textbf{93.80} & \textbf{92.43} & \textbf{88.07} & \textbf{74.30} \\
{ASP} & \underline{92.77} & \underline{92.40} & \underline{90.03} & \underline{85.60} & 70.47 \\
{Neural} & 84.07 & 84.17 & 83.23 & 80.87 & \underline{73.30} \\
\midrule
{Digit Acc.} & 99.25 & 99.24 & 99.12 & 98.58 & 96.89 \\
\bottomrule
\end{tabular}
\caption{Visual sudoku accuracy (\%) with different digit classifier training sizes. Best score is bolded, and the second best is underlined. The row of Digit Acc. shows the prediction accuracy on MNIST (digits 0-4) of the backbone neural model.}
\label{tab:awa-results}
\end{table}

\subsubsection{Results}
Table~\ref{tab:awa-results} shows the percentage of problems where all digits for each cell were given correctly over 3 runs. In every setup, we can see that \neuralmodelname performs the best, with constant gains against the second best method. 
An interesting trend that can be observed is that when the amount of data goes up, the ASP baseline performs much better. This is because when the predictions are correct, a complete reasoning mechanism will always get the correct solution.

On the other hand, while increasing the performance of the digit classifier does not have too dramatic of an impact for the neural method, we can see that the deterioration in performance is much smaller. This shows that neural models in general are very strong with imperfect inputs, showing merits over classical methods that completely plummet at 1\%.
\neuralmodelname is able to achieve the best of both worlds: when the predictions are accurate, \neuralmodelname can properly reason on them to get the correct solution, and when the predictions are less accurate, it uses the information to find a solution that is most compatible with the predictions.

\section{Related Work}
\label{sec:related-work}

\subsection{ASP Solvers}
Classical ASP systems such as {smodels}~\cite{DBLP:conf/lpnmr/NiemelaS97} and {DLV}~\cite{DBLP:journals/tocl/LeonePFEGPS06} perform 2-valued branching search by first examining the true case and then the false case (or vice-versa).
Later systems such as {clingo}~\cite{DBLP:conf/iclp/GebserKKOSW16} and {WASP}~\cite{DBLP:conf/lpnmr/AlvianoDFLR13} adopt CDCL-inspired architectures that combine deterministic propagation with branching decisions, allowing for more efficient search and learning from conflicts.

\modelname differs from these systems in that it focuses on deciding only the falsity of atoms and propagating the true atoms.
Furthermore, this design treats stability as a result of computation: when the solution is contradiction-free, stability is guaranteed (Theorem~\ref{thm:termination-stable}). \edit{These features enable a differentiable extension by NDProp, whose crisp case provably retains the same theoretical properties as DProp. Such differentiable extensions are not easily achievable with existing ASP solvers.}

Fuzzy answer set programming is an extension of ASP to handle degrees of truth using fuzzy logic \cite{DBLP:conf/jelia/NieuwenborghCV06,DBLP:journals/tplp/AlvianoP13}. These approaches adapt the stable model semantics to allow the use of fuzzy truth values, often employing t-norms to define rule satisfaction and model construction.

\subsection{Fixpoint Semantics}\label{subsec:rw_fixpoint}

A wide range of fixpoint semantics for ASP 
have been proposed in the literature.
Earlier works formalized nondeterministic fixpoint computation in various ways.  For example, \citeNBYB{DBLP:conf/pods/SaccaZ90} incorporate a backtracking mechanism into bottom-up search, and 
\citeNBYB{DBLP:journals/jlp/InoueS96} characterize a parallel procedure to compute stable models \cite{DBLP:conf/cade/InoueKH92,DBLP:conf/pricai/ShiraiH04} by a fixpoint operator with case splitting and constraints.
Efficient search strategies to limit the number of case splitting during fixpoint computation could be learned by machine learning techniques, but there has been no such attempt yet.

\citeNBYB{DBLP:journals/ngc/Fages91} shows a non-contradictory strategy for justification-based derivation like truth maintenance system yields a stable model at a fixpoint.  
Similarly, the notion of persistent computation by \citeNBYB{DBLP:journals/ai/LiuPST10} selects applicable rules to compute stable models, and the idea leads to ASP solvers GASP \cite{DBLP:journals/fuin/PaluDPR09} and ASPeRiX \cite{DBLP:journals/tplp/LefevreBSG17}.
Such a selection strategy is similar to \modelname, which captures stable models as computations of iterations using propagation and choice. The main difference is that in GASP and ASPeRix, the choice is performed on all negative body literals of a chosen rule, while \modelname performs decisions on any undecided atom; the latter eases extension into the continuous domain.

Several works characterize stable models via alternating fixpoints and approximation operators~\cite{DBLP:journals/tcs/Fitting02}.  The alternating fixpoint theory \cite{DBLP:journals/ai/Przymusinski91} defines stable models as fixpoints of an operator that alternates between deriving positive and negative information. 
\citeNBYB{Denecker2000} develop an approximation theory that captures various non-monotonic semantics including stable models, through the use of approximating operators and their fixpoints.  These frameworks yield stable models in different ways than \modelname, which directly produces stable models through propagation--decision steps.

\subsection{Neural Answer Set Solving}
{Neurosymbolic integration of ASP is a field being actively explored.
NeurASP~\cite{DBLP:conf/ijcai/YangIL20} finds stable models of programs that contain neural atoms using symbolic solvers, and uses the results to train the neural perception model.
SLASH~\cite{DBLP:journals/jair/SkryaginODK23} extends this system to support joint distributions, richer queries, and scalable inference.}

{In contrast to works that use classical ASP solvers, several works} have explored the use of differentiable approaches for solving ASP problems.
One approach to accomplish this is by using deep learning with discrete solvers such as clingo.
Differentiable ASP/SAT~\cite{DBLP:conf/ilp/Nickles18a} compute stable models with ASP solvers utilizing derivatives of a cost function, and \citeNBYB{DBLP:conf/lpnmr/DodaroIOR22} explore the use of deep learning to generate domain-specific solvers for WASP.

Another differentiable approach is to directly compute supported or stable models in vector spaces.
\edit{\citeNBYB{DBLP:conf/kr/AspisBR020} use a matrix representation of logic programs to compute supported models, and solve a root-finding problem using Newton’s method.  \citeNBYB{DBLP:conf/lpnmr/TakemuraI22} propose another matricized method by defining a cost function that yields a supported model when the cost is minimized to be zero.
\citeNBYB{DBLP:journals/corr/abs-2306-06821} extend such matricized methods by encoding loop formulas in vector spaces, so that stable models can be obtained by minimizing a single objective.
\citeNBYB{DBLP:conf/cilc/IeloR21} use GNNs to solve negative two literal programs, which are known to have the property that all supported models are stable models. GNNs have been used in Boolean constraint solvers other than ASP, such as Boolean Satisfiability~\cite{DBLP:conf/iclr/SelsamLBLMD19} and Maximum Satisfiability~\cite{moriyama2025graphbased}, with the latter specifically showing that t-norms are beneficial for developing differentiable solvers.
Our approach is classified as a differentiable approach as well, with the main differences being how the solutions are being computed: prior works use the cost function to optimize the variable assignments, while \neuralmodelname uses a propagation-decision process to search for stable models, with the decision heuristics being learned. Additionally, our work requires only binarization to get stable models rather than supported models, while prior works require the use of loop formulas or do not consider it at all.}

\section{Conclusion}

We have presented decision--propagation (\modelname), a novel method for computing stable models.
\modelname alternates between falsity decisions and truth propagations, and a successful \modelname computation is shown to capture the stable model semantics.
We then proposed \neuralmodelname, a differentiable extension of \modelname that uses neural decisions and fuzzy propagations, and show that it coincides with \modelname at the binary limit.
Experiments on random logic programs and neuro-symbolic benchmarks show that \neuralmodelname is able to learn to solve ASP problems, and can be combined with neural models to train end-to-end.
Compared to existing approaches that use symbolic reasoning, our approach incorporates learning in the reasoning step as well, resulting in a higher accuracy.
The results display the power of a differentiable and learnable solver in pipelines where reasoning is required.

Throughout the experiments, we have noticed that some t-norms work better than others in each setting. An extensive study of t-norms in particular settings is left out for future work, as each t-norm has different properties that are beneficial in different settings~\cite{eiter2026tnorm}.
Another direction to explore is whether the semantics of \neuralmodelname aligns with FASP. At the moment, \neuralmodelname is aimed at normal logic programs, but with modifications, it could be used to solve fuzzy variants.
Finally, extending the framework to support often used constructs such as constraints, choices and aggregates is left for future work. \edit{This would allow for use in more complex neuro-symbolic tasks, and would be an important step towards applying the framework to real-world scenarios.}

\clearpage
\section*{Acknowledgements}
This work has been supported by JSPS KAKENHI Grant Number JP25K03190, JST CREST Grant Number JPMJCR22D3 and Grant-in-Aid for JSPS Fellows Grant Number JP26KJ1236, Japan. This research was funded in whole or in part by the Austrian Science Fund (FWF) 10.55776/COE12. The authors thank colleagues in the KBS Group at TU Wien for helpful discussions.

\bibliographystyle{named}
\bibliography{references}

\input{appendix}

\end{document}

%% file: appendix.tex
\clearpage
\appendix
\onecolumn
\setcounter{theorem}{0}
\setcounter{lemma}{0}
\setcounter{definition}{0}
\setcounter{proposition}{0}

\section{Example Iteration of \modelname}
Consider the program $P$ with two rules:
\[
  a \leftarrow {not\,} b
  \qquad
  b \leftarrow {not\,} a.
\]
There are two stable models, $\{a\}$ and $\{b\}$. A \modelname run starts at
\[
(\Tr^0,\Fa^0,\Un^0)=(\emptyset,\emptyset,\{a,b\}).
\]
If the decision step falsifies $b$, then
\[
\Fa^1=\{b\},\qquad \Tr^1=C_{P,\mathcal{A}\setminus\Fa^1}^\infty(\Tr^0)=\{a\},\qquad
\Un^1=\mathcal{A}\setminus(\Tr^1\cup\Fa^1)=\emptyset,
\]
so the state is $(\Tr^1,\Fa^1,\Un^1)=(\{a\},\{b\},\emptyset)$, which is a stable model. If instead we falsify $a$, then
\[
\Fa^1=\{a\},\qquad \Tr^1=C_{P,\mathcal{A}\setminus\Fa^1}^\infty(\Tr^0)=\{b\},\qquad \Un^1=\emptyset.
\]
This illustrates how decisions grow $\Fa$ while propagation computes $C_{P,\mathcal{A}\setminus\Fa}^\infty$.
In the binary regime, NDProp follows the same trajectory: a binary decision selects one atom to falsify, and exact propagation
converges to the corresponding stable model.

\section{Experiment Specifics}
\label{app:exp}
All code is provided in the anonymized repository: \url{https://github.com/sotam2369/NDProp}.

\subsection{Hardware Specifics}
\begin{itemize}[leftmargin=*,itemsep=0.5pt,topsep=0.5pt]
    \item CPU: AMD Ryzen Threadripper PRO 5995WX (64 cores, 128 threads).
    \item GPUs: 2$\times$ NVIDIA RTX 6000 Ada (49~GB), 2$\times$ NVIDIA RTX A6000 (49~GB). All experiments use only a single GPU.
    \item Driver/CUDA: 560.35.05 / 12.6.
\end{itemize}

\subsection{NDProp Settings}
Table~\ref{tab:ndprop-settings} summarizes the NDProp hyperparameters used in each experiment.
\begin{table}[ht]
\centering
\small
\setlength{\tabcolsep}{6pt}
\renewcommand{\arraystretch}{1.1}
\begin{tabular}{lccccccc}
\toprule
\textbf{Experiment} & \textbf{Epochs} & \textbf{Iterations} & \textbf{Hidden Dim.} & \textbf{Logical Dim.} & \textbf{Learning Rate} & \textbf{Batch Size} & \textbf{T-norm} \\
\midrule
Random Logic Programs & 1000 & 10 & 32 & 32 & 0.001 & 64 & G\"odel \\
Neuro-Symbolic MNIST & 10 & 1 & 8 & 2 & 0.001 & 64 & Product \\
Visual Sudoku & 100 & 3 & 8 & 2 & 0.001 & 64 & G\"odel \\
\bottomrule
\end{tabular}
\caption{\neuralmodelname settings used for each experiment.}
\label{tab:ndprop-settings}
\end{table}

\subsection{Datasets}
Table~\ref{tab:dataset-sizes} reports the train/test split sizes for the MNIST arithmetic and membership tasks.
\begin{table}[ht]
\centering
\small
\setlength{\tabcolsep}{6pt}
\renewcommand{\arraystretch}{1.1}
\begin{tabular}{lcccccc}
\toprule
\textbf{Split} & \textbf{Addition} & \textbf{Addition2} & \textbf{Add2x2} & \textbf{Apply2x2} & \textbf{Member3} & \textbf{Member5} \\
\midrule
\textbf{Train} & 30000 & 15000 & 50000 & 10000 & 10000 & 10000 \\
\textbf{Test} & 5000 & 1000 & 1000 & 1000 & 1000 & 1000 \\
\bottomrule
\end{tabular}
\caption{Train/test split sizes for MNIST arithmetic and membership tasks.}
\label{tab:dataset-sizes}
\end{table}

\subsubsection{Addition}
This task takes two MNIST digit images and asks for their sum. Supervision is only the total, so the model must infer both digits consistently; the feasible sum range is $0\ldots18$.

\subsubsection{Addition2}
This task forms two two-digit numbers from four images and supervises only their total. The structure couples all four digits through place value, with sums restricted to $0\ldots198$.

\subsubsection{Add2x2}
This variant places four digits in a $2\times2$ grid and provides the sums of each row and column. These overlapping constraints force global consistency across all four digits, not just a single aggregate total.

\subsubsection{Apply2x2}
This task pairs three digits with a grid of operators and constrains the row and column results of the resulting expressions. The supervision ties together arithmetic structure and operator selection rather than a single sum.

\subsubsection{Member3}
This task presents three digit images and a query digit, and the label indicates whether the query appears at least once. The supervision is a single membership bit, so the model must reason over all positions jointly.

\subsubsection{Member5}
This task is the same as Member3 but with five images, increasing the search space and the number of possible matches. The label still provides only a single membership bit.

\subsubsection{Visual Sudoku}
For the Visual Sudoku 4x4 task, we generated instances that have clues ranging from 4-12, and each board is checked to have only a single solution.
3 examples of sudoku boards that are generated are shown in Figure~\ref{fig:sudoku_eg}.

\begin{figure}[h]
    \centering
    \begin{subfigure}[b]{0.3\linewidth}
        \centering
        \includegraphics[width=\linewidth]{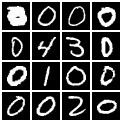}
        \caption{4 Clues}
    \end{subfigure}
    \hfill
    \begin{subfigure}[b]{0.3\linewidth}
        \centering
        \includegraphics[width=\linewidth]{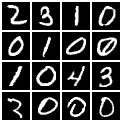}
        \caption{8 Clues}
    \end{subfigure}
    \hfill
    \begin{subfigure}[b]{0.3\linewidth}
        \centering
        \includegraphics[width=\linewidth]{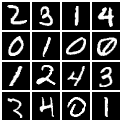}
        \caption{12 Clues}
    \end{subfigure}
    \caption{Example instances of 4x4 visual sudoku with different numbers of clues.}
    \label{fig:sudoku_eg}
\end{figure}

\subsection{ASP Encodings}
This appendix lists the ASP programs used in the experiments for \neuralmodelname.
NeurASP uses the original ASP programs provided by their authors, and SLASH uses the original for those provided, and NeurASP's for others.

\subsubsection{Addition}
\begin{lstlisting}
csum(S) :- digit(0, A), digit(1, B), S = A+B, S = 0..18.
\end{lstlisting}
\subsubsection{Addition2}
\begin{lstlisting}
sum(S) :- digit(0, A), digit(1, B), digit(2, C), digit(3, D),
          S = (A + C) * 10 + B + D, S = 0..198.
\end{lstlisting}

\subsubsection{Addition}
\begin{lstlisting}
row(0,S) :- digit(0,A), digit(1,B), S = A + B, S = 0..18.
row(1,S) :- digit(2,A), digit(3,B), S = A + B, S = 0..18.
col(0,S) :- digit(0,A), digit(2,B), S = A + B, S = 0..18.
col(1,S) :- digit(1,A), digit(3,B), S = A + B, S = 0..18.
\end{lstlisting}

\subsubsection{Apply2x2}
\begin{lstlisting}
apply(D1,0,D2,0,D3,R) :- digits(D1,D2,D3), R = (D1 + D2) + D3, R = -100..1000.
apply(D1,0,D2,1,D3,R) :- digits(D1,D2,D3), R = (D1 + D2) - D3, R = -100..1000.
apply(D1,0,D2,2,D3,R) :- digits(D1,D2,D3), R = (D1 + D2) * D3, R = -100..1000.

apply(D1,1,D2,0,D3,R) :- digits(D1,D2,D3), R = (D1 - D2) + D3, R = -100..1000.
apply(D1,1,D2,1,D3,R) :- digits(D1,D2,D3), R = (D1 - D2) - D3, R = -100..1000.
apply(D1,1,D2,2,D3,R) :- digits(D1,D2,D3), R = (D1 - D2) * D3, R = -100..1000.

apply(D1,2,D2,0,D3,R) :- digits(D1,D2,D3), R = (D1 * D2) + D3, R = -100..1000.
apply(D1,2,D2,1,D3,R) :- digits(D1,D2,D3), R = (D1 * D2) - D3, R = -100..1000.
apply(D1,2,D2,2,D3,R) :- digits(D1,D2,D3), R = (D1 * D2) * D3, R = -100..1000.

row(0,R) :- digits(D1,D2,D3), op(0,O1), op(1,O2), apply(D1,O1,D2,O2,D3,R).
row(1,R) :- digits(D1,D2,D3), op(2,O3), op(3,O4), apply(D1,O3,D2,O4,D3,R).

col(0,R) :- digits(D1,D2,D3), op(0,O1), op(2,O3), apply(D1,O1,D2,O3,D3,R).
col(1,R) :- digits(D1,D2,D3), op(1,O2), op(3,O4), apply(D1,O2,D2,O4,D3,R).
\end{lstlisting}

\subsubsection{Member3}
\begin{lstlisting}
member(D, 0) :- digit(0, N1), digit(1, N2), digit(2, N3), check(D),
                D != N1, D != N2, D != N3.
member(D, 1) :- check(D), not member(D, 0).
\end{lstlisting}

\subsubsection{Member5}
\begin{lstlisting}
member(D, 0) :- digit(0, N1), digit(1, N2), digit(2, N3), digit(3, N4), 
                digit(4, N5), check(D), D != N1, D != N2, D != N3, D != N4, D != N5.
member(D, 1) :- check(D), not member(D, 0).
\end{lstlisting}

\subsubsection{Visual Sudoku (4x4)}
We use the standard 4x4 Sudoku ASP encoding with normal rules only. The program includes per-cell exclusivity rules that ensure
each cell picks exactly one value, row/column exclusivity rules enforcing each digit appears once per row and column, and 2x2 subgrid
exclusivity rules. We also use integrity-style rules (shown as `bad` atoms) to rule out duplicate digits within a row, column, or block.
\begin{lstlisting}
% Per-cell exclusivity (example cell)
digit(1,1,1) :- not digit(1,1,2), not digit(1,1,3), not digit(1,1,4).
digit(1,1,2) :- not digit(1,1,1), not digit(1,1,3), not digit(1,1,4).
...

% Row/column exclusivity (examples)
digit(1,1,1) :- not digit(1,2,1), not digit(1,3,1), not digit(1,4,1).
digit(1,1,1) :- not digit(2,1,1), not digit(3,1,1), not digit(4,1,1).
...

% 2x2 subgrid exclusivity (example block)
digit(1,1,1) :- not digit(1,2,1), not digit(2,1,1), not digit(2,2,1).
digit(1,2,1) :- not digit(1,1,1), not digit(2,1,1), not digit(2,2,1).
...

% Constraints to rule out multiple digits in the same row/column/block
bad :- digit(1,1,1), digit(1,2,1), not bad.
bad :- digit(1,1,1), digit(1,3,1), not bad.
...
\end{lstlisting}